\documentclass[a4paper,11pt]{article}

\usepackage{amsmath}
\usepackage{amsfonts}
\usepackage{amssymb}
\usepackage{amsthm}
\usepackage{caption}
\usepackage[english]{babel}
\usepackage{fontenc}
\usepackage{graphicx}
\usepackage{mathrsfs}
\usepackage{bbm}
\usepackage{array}
\usepackage{float}
\usepackage{enumerate}
\usepackage{textcomp}
\usepackage{algorithm}
\usepackage{algorithmic}
\usepackage[font = scriptsize]{subcaption}
\usepackage{rotating}
\usepackage{color}
\usepackage{hyperref}
\usepackage{a4wide}

\usepackage{natbib}

\theoremstyle{plain}
\newtheorem{thm}{Theorem}
\newtheorem{lem}[thm]{Lemma}

\newtheorem{cor}[thm]{Corollary}

\theoremstyle{definition}

\theoremstyle{remark}

\numberwithin{equation}{section}

\newcommand{\ind}{\mathbbm{1}}
\newcommand{\pr}{\mathbb{P}}
\newcommand{\R}{\mathbb{R}}
\newcommand{\E}{\mathbb{E}}

\newcommand{\ceil}[1]{\left\lceil #1 \right\rceil}

\newcommand{\Var}{\mathrm{Var}}

\newcommand{\indist}{\stackrel{d}{\to}}
\newcommand{\inprob}{\stackrel{p}{\to}}

\newcommand{\pa}{\mathrm{pa}} 
\newcommand{\ch}{\mathrm{ch}} 
\newcommand{\an}{\mathrm{an}} 
\newcommand{\de}{\mathrm{de}} 
\newcommand{\depth}{\mbox{depth}}

\title{Random Intersection Trees}
\author{Rajen Dinesh Shah and Nicolai Meinshausen\\
Statistical Laboratory, University of Cambridge \\
Department of Statistics, University of Oxford\\
r.shah@statslab.cam.ac.uk and meinshausen@stats.ox.ac.uk}

\begin{document}
\maketitle
\begin{abstract}
Finding interactions between variables in large and high-dimensional datasets is often a serious computational challenge. Most approaches build up interaction sets incrementally, adding variables in a greedy fashion. The drawback is that potentially informative high-order interactions may be overlooked.
Here, we propose at an alternative approach for classification problems with binary predictor variables, called \emph{Random Intersection Trees}.  It works by starting with a maximal interaction that includes all variables, and then gradually removing variables if they fail to appear in randomly chosen observations of a class of interest.
We show that informative interactions are retained with high probability, and the computational complexity of our procedure is of order $p^\kappa$ for a value of $\kappa$ that can reach values as low as 1 for very sparse data; in many more general settings, it will still beat the exponent $s$ obtained when using a brute force search constrained to order $s$ interactions.
In addition, by using some new ideas based on min-wise hash schemes, we are able to further reduce the computational cost.
Interactions found by our algorithm can be used for predictive modelling in various forms, but they are also often of interest in their own right as useful characterisations of 
what distinguishes a certain class from others.
\bigskip
  
\end{abstract}

\section{Introduction} \label{section:intro}
In this paper, we consider classification with high-dimensional binary predictors. We suppose we have data that can be written in the form $(Y_i, X_i)$ for observations $i =1, \ldots, n$; $Y_i$ is the class label and $X_i \subseteq \{1,\ldots, p\}$ is the set of active predictors for observations $i$ (out of a total of $p$ predictors). An important example of this type of problem is that of text classification, where then $X_i$ is the set of frequently appearing words (in a suitable sense) for document $i$, and $Y_i$ indicates whether the document belongs to a certain class. In this case, the dimension $p$ can be of the order of several thousand or more. More generally, if data with continuous predictors are available, they can be converted to binary format by choosing various split-points, and then reporting whether or not each variable exceeds each of these thresholds.

Our aim here is to develop methodology that can discover important interaction terms in the data without requiring that any of their lower order interactions are also informative. More precisely, we are interested in finding subsets $S\subset\{1,\ldots,p\}$ of all predictor variables that occur more often for observations in a class of interest than for other observations. We will use the terms ``leaf nodes'', ``rules'', ``patterns'' and ``interactions'' interchangeably to describe such subsets $S$.
For simplicity, suppose there are only two classes, the set of labels being $\{0, 1\}$. The case with more than two classes can be dealt with using one-versus-one, or one-versus-all strategies. Given a pair of thresholds, $0 \leq \theta_0 < \theta_1 \leq 1$, our goal is to find all sets $S$ (or as many as possible), for which 
\begin{equation} \label{eq:opt}
\pr_n(S\subseteq X|Y=1) \ge \theta_1 \quad \mbox{and} \quad \pr_n(S\subseteq X|Y=0)\le \theta_0 .
\end{equation}
Here and throughout the paper, we use the subscript $n$ to indicate that the probabilities are empirical probabilities. For example, for $c \in \{0, 1\}$,
\[
 \pr_n(S\subseteq X|Y=c) := \frac{1}{|C_c|} \sum_{i \in C_c} \ind_{\{S \subseteq X_i\}} ,
\]
where we have denoted the set of observations in class $c$ by $C_c$. Of course, one would also be interested in sets $S$ which satisfy a version of \eqref{eq:opt} with classes 1 and 0 interchanged, but we will only consider \eqref{eq:opt} for simplicity.

The interaction terms uncovered can be used in various ways. For example, they can be built into tree-based methods, or form new features in linear or logistic regression models.  The interactions may also be of interest in their own right, as they can characterise distinctions between classes in a simple and interpretable way. These potentially high-order interactions that our method aims to target would be very difficult to discover using existing methods, as we now explain.

A pure brute force search examines each potential interaction $S$ of a given size to check whether it fulfills \eqref{eq:opt}. Restricting the order of interactions to size $s$, the computational complexity scales as $p^s$, rendering problems with even moderate values of $p$ infeasible.

Instead of searching through every possible interaction, tree-based methods build up interactions incrementally. A typical tree classifier such as \emph{CART} \citep{CART} works by building a decision tree greedily from root node to the leaves; see 
also \citet{loh1997split}. The feature space is recursively partitioned based on the variable whose presence or absence best distinguishes the classes. The myopic nature of this strategy makes it a computationally feasible approach, even for very large problems. The downside is that it produces rather unstable results and hence gives poor predictive performance. Moreover, because of the incremental way in which interactions are constructed, the success of this strategy in recovering an important interaction $S$ rests on at least some of its lower order interactions being informative for distinguishing the classes.

Approaches based on tree ensembles can somewhat alleviate the problem of tree instability; \emph{Random Forests} \citep{breiman01random} is a prominent example. Here the data with which the decision trees are constructed is sampled with replacement from the original data. Further randomness is introduced by randomising over the subset of variables considered for each split in the construction of the trees. While the results of \emph{Random Forests} are very complex and hard to interpret, one can examine what are known as variable importance measures. These aim to quantify the marginal or pairwise importance of predictor variables \citep{strobl2008conditional}. Though such measures can be useful, they may fail to highlight important high-order interactions between variables.

More recently, there has been interest in algorithms that start from deep splits or leaf nodes in trees and then try to build a simpler model out of many thousands of these leaves by regularisation and dimension reduction. Examples include \emph{Rule Ensembles} \citep{friedman2005plv}, \emph{Node Harvest} 
\citep{meinshausen2010node} 
and the general framework of \emph{Decision Lists} \citep{marchand2006learning,rivest1987learning}. Though these methods have been demonstrated to improve on \emph{Random Forests} in some situations, they nevertheless crucially rely on a good initial basis of leaf nodes. These bases are usually generated by tree ensemble methods and so, if the base trees miss some important splits, they would also be absent in the results of these derivative algorithms.

A complementary approach has developed in data mining under the name of frequent itemset search, starting with the \emph{Apriori} algorithm \citep{agrawal1994fast}, which has since then developed into many improved and more specialised forms. The starting point for these was ``market basket analysis'', where the shopping behaviour of customers is analysed and the goal is to identify baskets that are often bought together.
While generally very successful, these methods work on the principle that subsets of frequent itemsets are also frequent. Thus if there is no strong marginal effect of any of the variables that is involved in a decision rule, there is no advantage in general compared to a brute force search.

We now give a simple example where tree-based approaches and those based on the \emph{Apriori} algorithm will struggle. Suppose our data are independent realisations of the pair of random variables $(X, Y)$, whose distribution is given as follows. Let $Z$ be the random binary vector
 \[
  Z = (\ind_{\{1 \in X\}}, \ldots, \ind_{\{p \in X\}})^T.
 \]
Suppose that, conditional on $\{Y=0\}$, $Z_k$ for $k=1,\ldots,p$ are independent; and conditional on $\{Y=1\}$, $Z_k$ for $k=2,\ldots,p$ are independent and $Z_1 = Z_2$. We take the marginal distributions of the $Z_k$ to all be Bernoulli($q_Z$). Finally, let $Y$ be independent with a Bernoulli($q_Y$) distribution. Then the interaction $S=\{1,2\}$ is certainly important for distinguishing the classes. However, any individual variable in $\{1,\ldots,p\}$ does not appear more or less frequently on average among class~1 compared to class~0. This lack of any marginal relationship between the class label and the first two predictors would cause tree-based methods to perform poorly. In addition, using the \emph{Apriori} algorithm or the brute force method to find $S$ would require computational cost of the order $p^2$.

This paper looks at a new way to discover interactions, which we call \emph{Random Intersection Trees}. Rather than searching through potential interactions directly, our method works by looking for collections of observations whose common active variables together form informative interactions. We present a basic version of the \emph{Random Intersection Trees} algorithm in the following section. This approach allows for computationally feasible discovery of interactions in settings where most existing procedures would perform poorly. Bounds on the complexity of our algorithm are given in Section~\ref{section:computational}.
For example, our results yield that in the scenario discussed in the previous paragraph, the order of computational complexity of our method is at most $o(p^\kappa)$ for any $\kappa >1$. In Section~\ref{section:stopping}, we propose some modifications of our basic method to reduce its computational cost, based on min-wise hash schemes. Some numerical examples are given in Section~\ref{section:numerical}. We conclude with a brief discussion in Section~\ref{section:discussion}, and all technical proofs are collected in the appendix.

\section{Random Intersection Trees} \label{section:random}
Our method searches for important interactions by looking at intersections of randomly chosen observations from class~1. We start with the full set of variables, and then remove those that are not present in the observations chosen. If a pattern $S$ has high prevalence in class~1, i.e.~$\pr_n (X=S |Y=1)$ is large, it will be included in the observations chosen with high probability. Thus, provided the overall process is repeated often enough, $S$ is likely to be retained in at least one of the final intersections. One could then consider each of these intersections as possible solutions of \eqref{eq:opt}, checking whether their prevalence among class 0 is below $\theta_0$.

Arranging the procedure in a tree-type search makes the algorithm more computationally efficient. To describe the details, we first define some terms associated with trees that will be needed later. Recall that a tree is a pair $(N, E)$ of nodes and edges forming a connected acyclic (undirected) graph. We will always assume (with no loss of generality) that $N = \{1, \ldots, |N|\}$.
A \emph{rooted tree} is the directed acyclic graph obtained from a tree by designating one node as root and directing all edges away from this root.

Let $\alpha$ and $\beta$ be two nodes in a rooted tree, with $\beta$ not the root node. If $(\alpha, \beta) \in E$, $\beta$ is said to be the \emph{child} of $\alpha$, and  $\alpha$, the \emph{parent} of $\beta$. We will denote by $\ch(\alpha)$, the set of children of a node $\alpha$. Since we are only considering rooted trees here as opposed to general directed graphs, we will differ with convention slightly and will use $\pa(\beta)$ to mean the unique parent of $\beta$. Thus here, $\pa(\beta)$ is a node itself, whereas $\ch(\alpha)$ is a set of nodes.

If $\alpha \neq \beta$ lies on the unique path from the root to $\beta$, we say $\alpha$ is an \emph{ancestor} of $\beta$, and $\beta$ is a \emph{descendant} of $\alpha$. The \emph{depth} of $\alpha$, denoted $\depth(\alpha)$, is the number of ancestors of $\alpha$: $\depth(\alpha) = |\an(\alpha)|$. In particular, the depth of the root node is $0$. The \emph{depth} (also known as the \emph{height}) of a rooted tree is the length of the longest path, or equivalently, the greatest number of ancestors of any particular node. By \emph{level} $d$ of the tree, we will mean the set of nodes with depth $d$.

We will say an indexing of the nodes is \emph{chronological} if, for every parent and child pair, larger indices are assigned to the child than the parent. In particular, the root node will be $1$. Note that both depth-first and breadth-first indexing methods are chronological in this way.

\begin{algorithm}
\caption{A basic version of Random Intersection Trees} \label{alg:1}
 \begin{algorithmic}
  \FOR{tree $m=1$ \TO $M$}
    \STATE Construct rooted tree $m$ of depth $D$. Each node in levels $0,\ldots,D-1$ has $B$ children and $B$ can be random. Let $J$ be the total number of nodes in the tree, and index the nodes in a chronological way.
    \STATE Set $S_1$ to be a randomly chosen observation from class~1.
    \FOR{node $j=2$ \TO $J$}
    \STATE
    Draw a random observation $i(j)$ from class~1. 
    \STATE Set $S_j = X_{i(j)} \cap S_{\pa(j)}$.
     \ENDFOR
   \STATE Denote the collection of resulting sets of all nodes at depth $d$, for $d=1, \ldots, D$, by \\$L_{d, m} = \{S_j: \;\depth(j)=d\}$.
      
   \ENDFOR
   \RETURN $L_D := \bigcup_{m=1} ^M L_{D, m}$.
 \end{algorithmic}
\end{algorithm}

Algorithm~1 describes a basic version of the \emph{Random Intersection Trees} procedure. We see that, each node in each tree is associated with a randomly drawn observation from class~1. For every tree, we visit each non-root node in turn, and compute the intersection of the observation assigned to it, and all those assigned to its ancestors. Because of the way the nodes are indexed, parents are always visited before their children, and this intersection can simply be computed as $S_j = X_{i(j)} \cap S_{\pa(j)}$. This is crucial to reducing the computational complexity of the procedure, as we shall see in the next section.

Each of the sets assigned to the leaf nodes of each of the trees yield a collection of potential candidate interactions, $L_D$. One could then proceed to test these as potential solutions to \eqref{eq:opt}; we present a more efficient approach in Section~\ref{section:stopping}, where we build this testing step into the construction of the trees. An illustration of this improved algorithm applied to the Tic-Tac-Toe data discussed in Section~\ref{section:numerical} is given in Figure~\ref{fig:tictactree}. Here, the root node contains a randomly drawn final win-state for black (class~1). This corresponds to $S_1$ in our algorithm. 
For each other node $j$, the randomly chosen additional black-win state $X_{i(j)}$ is shown along the edge from its parent node and the new intersection $S_j$ in the corresponding node. The early stopping that is added in the improved algorithm also allows to run until the algorithm has terminated in all nodes and no prior specification of the tree depth will thus be necessary in practice, as will be shown in Section~\ref{section:stopping}.

\begin{figure}
\begin{center}
\includegraphics[width=0.7\textwidth]{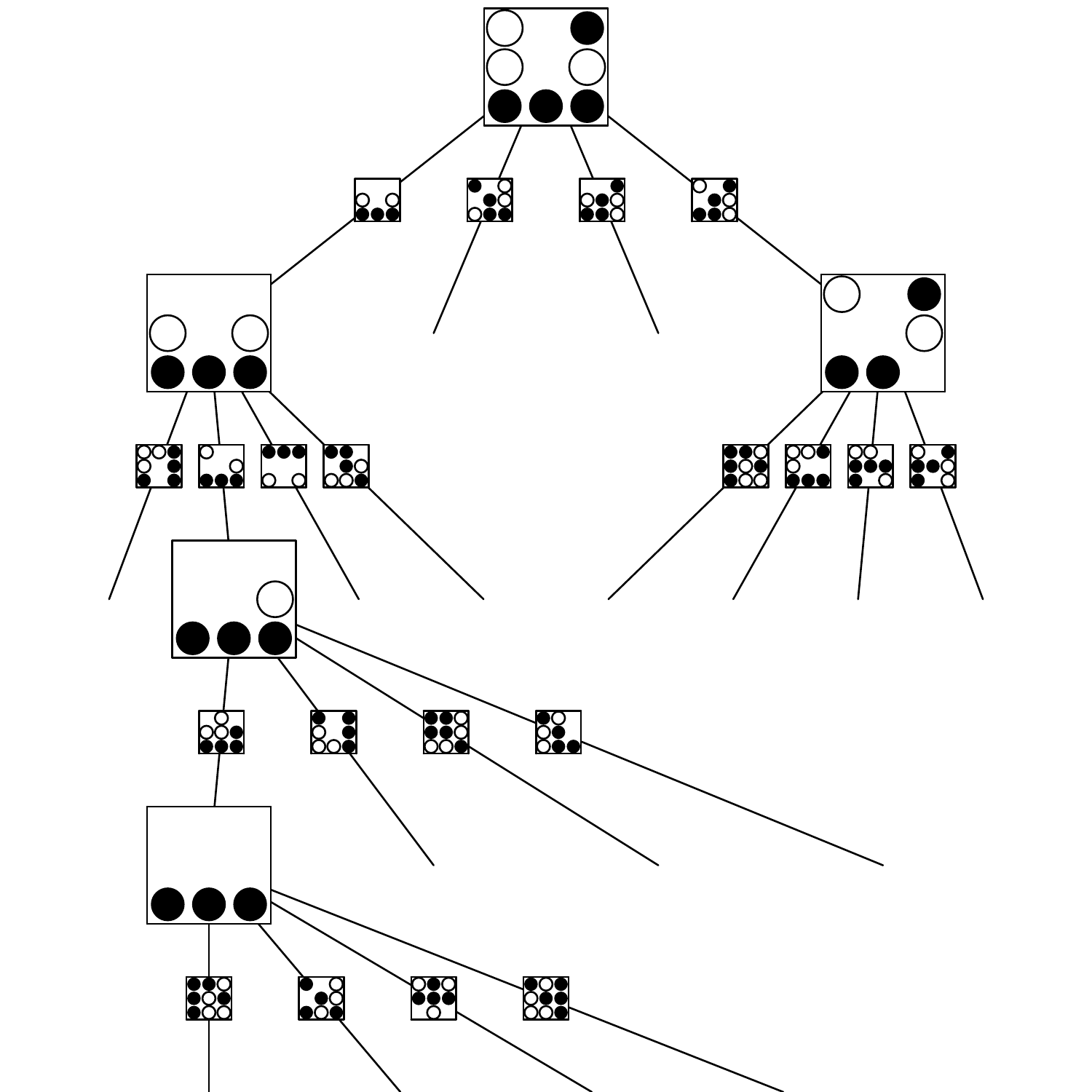}
\caption{ \label{fig:tictactree}  \textit{  An intersection tree. Starting with a randomly chosen class~1 (black wins) observation at the root node, $B=4$ randomly chosen class~1 observations are intersected with the pattern. These randomly chosen observations are shown along the edges and the resulting intersections $S_j$ as the nodes in the next layer of the tree. Nodes are only shown if the corresponding patterns $S_j$ have an estimated prevalence among class 0 below a set threshold; the branching of the tree terminates for all other nodes. The algorithm continues until all resulting $S_j$ corresponding to the leaf nodes have prevalence among class 0 exceeding the threshold. Here, one of the winning states for black is filtered out after three intersections.}}
\end{center}
\end{figure}

\section{Computational complexity}
\label{section:computational}
How many trees do we have to compute to have a very high probability of finding an interesting interaction $S$ that fulfills \eqref{eq:opt}? And what is the required size of these trees? If the interaction is not associated with a main effect, most approaches like trees and association rules would require of order $p^{|S|}$ searches. In this section, we show that in many settings, \emph{Random Intersection Trees} improves on this complexity. We consider a single interaction $S$ of size $s = |S|$, and examine the computational cost for returning $S$ as one of the candidate interactions, with a given probability. We will see that this depends critically on three factors:
\begin{itemize}
\item\textbf{Prevalence} $\theta_1 := \pr_{n}( S \subseteq X | Y=1)$ of the interaction pattern. If the pattern $S$ in question appears frequently in class~1, the search is more efficient.
\item\textbf{Sparsity} $\delta_k := \pr_n (k \in S |Y=1)$ of the predictor variables $k=1,\ldots,p$. If $\delta_k$ is very low for many $k$ (and sparsity of predictors consequently high), computation of the intersections is much cheaper, and so overall computational cost is greatly reduced. Indeed, for a fixed tree $m$, consider a node $j$ with depth $d<D$. We have that
\[
 \E(|S_j|) = \sum_{k =1} ^p \delta_k ^{d+1}.
\]
Thus, for $j' \in \ch(j)$, computation of $S_{j'}$ requires on average at most
\begin{equation*} 
 O\left(\log (p) \sum_{k =1} ^p \delta_k ^{d+1}\right)
\end{equation*}
operations. This is because in order to compute the intersection, one can check whether each member of $S_j$ is in $X_{i(j')}$, and each such check is $O(\log(p))$ if the sets $X_i$ are ordered so a binary search can be used. If we compare this to the $O(p)$ computations required to calculate each of the $S_j$ if no tree structure were used, we see that large efficiency gains are possible when $d \geq 1$ if many variables are sparse. For intersections with the root node, the tree structure offers no advantage, and in practice, branching the tree only after level 1 (so the root node has only one child), is more efficient, though this modification does not improve the order of complexity.
\item\textbf{Independence of }$S$: Define $\nu := \max_{k \in S^c} \pr_{n}(k \in X | S \subseteq X, Y=1)$. If $\nu$ is low, less computational effort is required to recover $S$. Note that if, for some $k \in S^c$, $\pr_{n}(k \in X | S \subseteq X) = 1$, interest would centre on $S \cup \{k\}$ rather than $S$ itself. Indeed, if $S$ satisfied \eqref{eq:opt}, so would $S \cup \{k\}$. In general, if $\nu$ is large, the search will tend to find sets containing $S$, though not necessarily $S$ itself.
\end{itemize}
With the assumptions that $\theta_1 > 0$ and $\nu < 1$, we can give a bound on the computational complexity of the basic version of \emph{Random Intersection Trees} introduced in the previous section.
\begin{thm}\label{thm:1}
 For suitable choices of $M$, $D$ and the distribution of $B$, the expected order of computations needed for $L_D$ to contain $S$ with probability $1 - \eta < 1$ is bounded above by the minimum over $\epsilon \in (0, 1]$ of
 \begin{equation}  \label{eq:result_1}
\log(1/\eta)\frac{\log^2 (p)}{\epsilon} \bigg\{p + \sum_{k:\,(1+\epsilon)\delta_k > \theta_1} p^{\frac{\log\{(1+\epsilon)\delta_k / \theta_1\}}{\log(1/\nu)}}  \bigg\} .
 \end{equation}
\end{thm} 
As a function of the number of variables $p$, there is a contribution of $p\log^2(p)$ and an additional contribution in the brackets that depends on the sparsity $\delta_k$ of each variable. Sparse variables do not contribute to this sum and the sum can be arbitrarily close to 1 if the sparsity among variables is high enough. This would yield a computational complexity with order bounded above by $o(p^\kappa)$ for any $\kappa>1$, compared to the corresponding complexity of $p^s$ for a brute force search. In most interesting settings, however, we would not achieve a nearly linear scaling in complexity, but would hope to still be faster than a brute force search.

\paragraph{The influence of sparsity on computational complexity.} 
It is interesting to make the influence of the sparsity of individual variables, $\delta_k$, on the overall computational complexity, more explicit. We have the following corollary to Theorem~\ref{thm:1}.
\begin{cor}\label{cor:1}
Define $\beta$ by $\nu= \theta_1^\beta$.
Suppose that $\gamma,\alpha^\star,\alpha_\star$ are such that $\alpha^\star > \alpha_\star$, and 
\begin{align*}
\delta_k &\;\le\;   \theta_1^{1-\alpha^\star} & \mbox{for all } j\in\{1,\ldots,p\} \\ 
\delta_k & \;>\;  \theta_1^{1-\alpha_\star} & \mbox{for at most } O(p^\gamma) \mbox{ variables.}
\end{align*}
For suitable choices of $M$, $D$ and the distribution of $B$, the expected order of computations needed for $L_D$ to contain $S$ with probability $1 - \eta < 1$ is bounded above by
 \begin{equation}
o\big(  p^\kappa  \big) \qquad\mbox{ for any } \kappa> \max\bigg\{ \frac{\alpha^\star}{\beta} + \gamma,\; \Big[\frac{\alpha_\star}{\beta}\Big]_+ +1\bigg\}.
 \label{eq:result_2}
 \end{equation}
\end{cor}
The implication of Theorem~\ref{cor:1} is most apparent if we take $\gamma=1$ as we can then set $\alpha_\star=0$. In this case, \[ \alpha^\star= 1- \frac{ \log(\max_k \delta_k)}{\log(\theta_1)}.\]
We can then bound the computational complexity by 
\begin{equation}\label{eq:boundfrac} o\big(  p^\kappa  \big) \qquad\mbox{ for any } \kappa> 1+\frac{\log(1/\theta_1)-\log(1/\max_k \delta_k)}{\log(1/\nu)}. \end{equation}
The fraction on the right-hand side is a function of the prevalence of the pattern $S$, $\theta_1$, the maximum sparsity of the variables, and the maximum sparsity of the variables in $S^c$, conditional on the presence of $S$.   As long as this fraction is less than 1, the computational complexity is guaranteed to be better than a brute force search with the knowledge that $s=2$, and the relative advantage grows for larger sizes of the pattern. 

\paragraph{Independent noise variables.}
To gain further insight, we consider the special case where variables in $S^c$ are independent of $S$ (conditional on being in class~1), in the sense that for all $k \in S^c$,
\begin{equation} \label{eq:indep}
 \pr_n (k \in X |S \subseteq X, Y=1) = \pr_n (k \in X | Y=1) = \delta_k.
\end{equation}
\begin{cor}\label{cor:2}
Assume \eqref{eq:indep} and that $\delta_k < 1$ for all $k$. Define $\tau:=\log(\theta_1) / \log(\max_k \delta_k)$. Then for suitable choices of $M$, $D$ and the distribution of $B$, the expected order of computations needed for $L_D$ to contain $S$ with probability $1 - \eta < 1$ is bounded above by
\begin{equation} o(p^\kappa) \qquad\mbox{ for any }\kappa > \tau.  \label{eq:result_3}
\end{equation}
\end{cor}
We see that the computational complexity is approximately linear in $p$ if the prevalence of the pattern $S$ is as high as the prevalence of the least sparse predictor variables. 

We can also consider the situation where in addition to the independence \eqref{eq:indep}, all variables have the same sparsity $\delta$.
If the prevalence $\theta_1$ of $S$ is only as high as that of a random occurrence of two independent predictor variables, we get $\tau=2$ and the computational complexity is quadratic in $p$. In this case, the algorithm would not yield a computational advantage over brute force search if looking for patterns of size 2. This is to be expected since \emph{every} pattern $S$ of size 2 would have the same prevalence in this scenario, and so there is nothing special about a pattern $S$ of size 2 with prevalence $\delta^2$, and in general no hope of beating the complexity $p^{s}$ of a brute force search.  However, the bound in (\ref{eq:result_3}) is independent of $s$. Thus provided the prevalence $\theta_1$ drops more slowly that the rate $\delta^{s}$, at which every pattern of size $S$ would occur randomly among independent predictor variables, our results show that \emph{Random Intersection Trees} is still to be preferred over a brute force search.

\section{Early stopping using min-wise hashing}
\label{section:stopping}
While Algorithm~1 is computationally attractive, the following observation suggests that further improvements are possible. Suppose that, for a particular tree, we have just computed the intersection $S_j$ corresponding to a node $j$ at depth $d<D$. If
\[
 \pr_n (S_j \subseteq X| Y = 0) > \theta_0, 
\]
then since for all $j' \in \de(j)$, $S_{j'} \subseteq S_j$, we also have
\[
 \pr_n (S_{j'} \subseteq X | Y = 0) > \theta_0.
\]
Thus no intersection sets corresponding to descendants of $j$ have any hope of yielding solutions to \eqref{eq:opt}, and so all further associated computations are wasted.

In view of this, one option would be to compute the quantity $\pr_n (S_j \subseteq X | Y=0)$ at each node $j$ as the algorithm progresses, and if this exceeds the threshold $\theta_0$, not visit any descendants $j'$ of $j$ for computation of $S_{j'}$. This could be prohibitively costly, though, as it would require a pass over all observations in class 0, for each node of each tree. One could work with a subsample of the observations, but if $\theta_0$ is low, the subsample size may need to be fairly large in order to estimate the probabilities to a sufficient degree of accuracy.

Instead, we propose a fast approximation, using some ideas based on min-wise hashing 
\citep{broder1998min,cohen2001finding,datar2002estimating} applied to the columns of the data-matrix. We describe the scheme by leaving aside the conditioning on $Y=0$, which can be added at the end by restricting to observations in class 0.  Consider taking a random permutation $\sigma$ of all observations $\{1,\ldots,n\}$. Let $h_\sigma(k)$ be the minimal value $\iota$ such that variable $k$ is active in observation $\sigma(\iota)$:
\[  h_\sigma(k) =\min \{\iota' : k \in X_{\sigma(\iota')}\} .  \]
It is well known \citep{broder1998min} that the probability that $h_\sigma(k)$ and  $h_\sigma(k')$ agree for two variables $k,k'$ under a random permutation $\sigma$ is identical to the Jaccard-index for the two sets $I_k=\{i : k\in X_i\}$ and $I_{k'}=\{i: k'\in X_i\}$, that is 
\[
\pr_\sigma(h_\sigma(k)=h_\sigma(k'))\;=\; \frac{|I_k\cap I_{k'}|}{|I_k\cup I_{k'}|}.
\]
Here the subscript $\sigma$ indicates that the probability is with respect to a random permutation $\sigma$ of the observations. A min-wise hash scheme is typically used to estimate the Jaccard-index 
by approximating the probability on the left-hand side of the equation above.

Now,
\begin{align*} 
\pr_n(S \subseteq X) & = \pr_n( k \in X \mbox{ for all } k \in S ) 
\\ & = \pr_n(k \in X \mbox{ for all } k \in S \, | \, \exists \, k' \in S \mbox{ such that } k' \in X ) \\
& \qquad  \times \pr_n(\exists k \in S \mbox{ such that } k \in X). 
\end{align*}
Let us denote the first and second terms on the right-hand side by $\pi_1 (S)$ and $\pi_2 (S)$ respectively.
Note that $\pi_1 (S)$ is equal to the probability that all variables $k \in S$ have the same min-wise hash value $h_\sigma(k)$:
\begin{equation} \label{eq:pi1}
 \pi_1 (S) = \pr_\sigma (\exists \, i : h_\sigma(k)=i \mbox{ for all } k\in S).
\end{equation}

Turning now to $\pi_2 (S)$, observe that
\begin{equation} \label{eq:expect_hash}
 \E_\sigma (\min_{k \in S} h_\sigma (k)) = \frac{n + 1}{\pi_2 (S) n + 1},
\end{equation}
and so
\begin{equation} \label{eq:pi2}
 \pi_2 (S)= \frac{n+1}{n} \left\{ \frac{1}{\E_\sigma (\min_{k \in S} h_\sigma (k))} - \frac{1}{n+1} \right\}.
\end{equation}
A derivation of \eqref{eq:expect_hash} is given in the appendix.

Equations \eqref{eq:pi1} and \eqref{eq:pi2} provide the basis for an estimator of $\pr_n(S\subseteq X)$. First we generate $L$ random permutations of $\{1, \ldots, n\}$: $\sigma_1, \ldots , \sigma_L$. We then use these to create an $L \times p$ matrix $H$ whose entries are given by
\[
 H_{l k} = h_{\sigma_l}(k).
\]
Now we estimate $\pi_1 (S)$ and $\pi_2 (S)$ by their respective finite-sample approximations,  $\hat{\pi}_1 (S) $ and $\hat{\pi}_2 (S) $:
\begin{gather*}
 \hat{\pi}_1 (L; S, H) := \tfrac{1}{L} \sum_{l=1}^L \ind_{\{ H_{l k}=H_{l k'} \text{ for all } k,k'\in S\}} \\
 \hat{\pi}_2 (L; S, H) := \frac{n+1}{n} \left\{ \frac{1}{\tfrac{1}{L}\sum_{l=1}^L  \min_{k\in S} H_{l k}} - \frac{1}{n+1} \right\}.
\end{gather*}
Finally, we estimate $\pr_n (S\subseteq X)$ by
\begin{equation} \label{eq:approx}
 \hat{\pr}_n (L; S, H) := \hat{\pi}_1 (S, H) \cdot \hat{\pi}_2 (S, H).
\end{equation}
To our knowledge, this use of min-wise hashing techniques, and in particular the estimator $\hat{\pi}_2 (S) $, is new. The estimator enjoys reduced variance compared to that which would be obtained using subsampling, as the following theorem shows.
\begin{thm} \label{thm:minhash_asymp} For $\hat{\pr}_n (L; S, H)$, $\pi_1(S)$ and $\pi_s(S)$ defined as in (\ref{eq:approx}), (\ref{eq:pi1}), and (\ref{eq:pi2}) respectively, 
 \begin{equation} \label{eq:minhash_asymp}
 \sqrt{L} (\hat{\pr}_n (L; S, H) - \pr_n(S\subseteq X))  \indist N(0, \pi_2 (S)  ^2 \pi_1 (S) (1 - \pi_1 (S) \pi_2 (S)) + o_n(1)).
\end{equation}
\end{thm}
A derivation is given in the appendix. Here the subscript $n$ in $o_n (1)$ is used to indicate that $n$ is tending to infinity. Comparing the variance of the normal distribution in \eqref{eq:minhash_asymp} to that which would be obtained if subsampling ($\pi_2 (S) \pi_1 (S) (1 - \pi_1(S) \pi_2 (S)) + o_n(1)$), we see that a factor of $\pi_2 (S)$ is gained: matching the accuracy of subsampling with the min-wise hash scheme would require roughly $1/\pi_2 (S)$ times as many samples.
By using min-wise hashing, choosing $L=100$ typically delivers a reasonable approximation as long as we just want to resolve values at $\theta_0=0.01$ and above.

An improved version of Algorithm 1, building in the ideas discussed above, is given in Algorithm 2 below. Note that $\hat{\pr}_n(S_{\pa(j)},H)$ need only be computed once for every $j$ with the same parent. 

\begin{algorithm}[h!]
\caption{Random Intersection Trees with early stopping}
 \begin{algorithmic}
 \STATE Compute the $L\times p$ min-wise hash matrix $H$, using only class 0 observations.
  \FOR{tree $m=1$ \TO $M$}
    \STATE Construct rooted tree $m$. Each node in levels $0,\ldots,D-1$ has $B$ children and $B$ can be random. Let $J$ be the total number of nodes in the tree, and index the nodes in a chronological way.
    \STATE Set $S_1$ to be a randomly chosen observation from class~1.
    \FOR{node $j=2$ \TO $J$}
    \IF{$\hat{\pr}_n(S_{\pa(j)},H) \le \theta_0$}
    \STATE
   Draw a random observation $i(j)$ from class~1. 
    \STATE Set $S_j = X_{i(j)} \cap S_{\pa(j)}$.
    \ENDIF
     \ENDFOR
   \STATE Denote the collection of resulting sets of all nodes at depth $d$, for $d=1, \ldots, D$, by \\$L_{d, m} = \{S_j: \;\depth(j)=d\}$.
   \ENDFOR
   \RETURN $L_D := \bigcup_{m=1} ^M L_{D, m}$.
 \end{algorithmic}
\end{algorithm}

Early stopping decreases the computational cost of the algorithm as many nodes in the trees generated may not need to have their associated intersections calculated. In addition, the set of candidate intersections $L_D$ will be smaller but the chance of it containing interesting intersections would not decrease by much. These gains comes at a small price, since the min-wise hash matrix $H$ must be computed, and the computational effort going into this will in turn determine the quality of the approximation in \eqref{eq:approx}. We have previously shown the complexity bounds in the absence of early stopping and thus avoided the difficulty of making this trade-off explicit. We will use the improved version of \emph{Random Intersection Trees} with early stopping in all the practical examples to follow, taking small values of $L$ in the range of a (few) hundred permutations.

The depth $D$ of the tree is still given explicitely in Algorithm~2. An interesting modification creates the tree recursively. Starting with the root node, $B$ children are added to all leaf nodes of the tree in which the  early stopping citerion has not been triggered yet. When the algorithm terminates, all intersection in the leaf nodes of the final tree are collected. 

\section{Numerical Examples}
\label{section:numerical}
In this section, we give two numerical examples to provide further insight into the performance of our method. The first is about learning the winning combinations for the well-known game Tic-Tac-Toe. This example serves to illustrate how \emph{Random Intersection Trees} can succeed in finding interesting interactions when other methods fail. The second example concerns text classification. Specifically, we want to find simple characterisations (using only a few words, or word-stems in this case) for classes within a large corpus in a large-scale text analysis application.

\subsection{Tic-Tac-Toe endgame prediction}

\begin{figure}
\begin{center}
\includegraphics[width=1\textwidth]{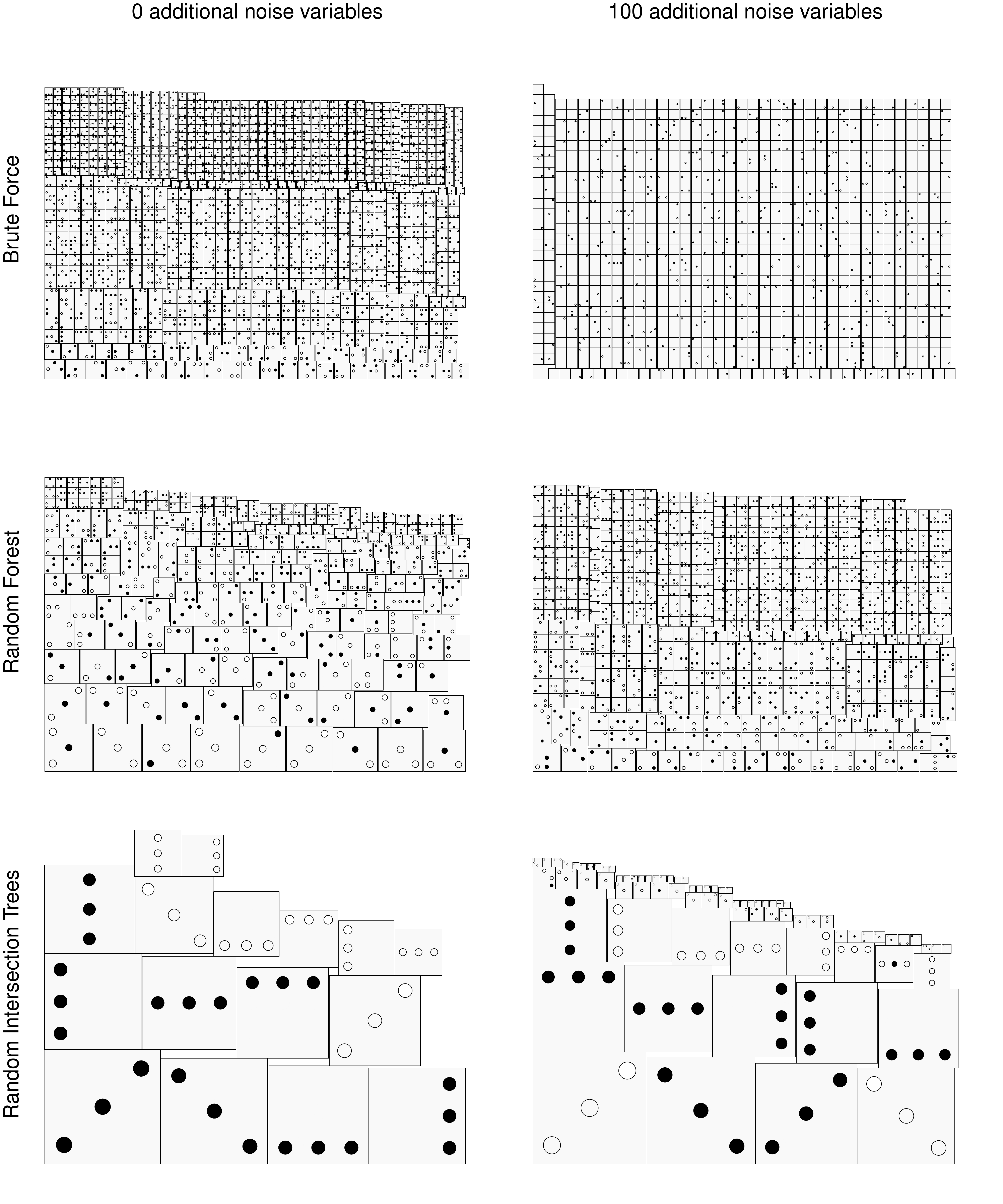}
\caption{ \label{fig:tictac}  \textit{The probability of choosing a given pattern with {Random Intersection Trees} (bottom row), {Random Forests} of depth 3 (middle row) and brute force search among all interactions of size 3 (top row) for the Tic-Tac-Toe data (left panel); and the same results in the case when 100 noise variables are added (right panel). Note that {Random Intersection Trees} were not constrained to find interactions of depth 3. The area of each pattern is proportional to the probability of being chosen. In the case with noise variables, some of the patterns with the very smallest areas also contained a small number of noise variables, which are not shown. Just counting three- to five-way interactions, there are more than $10^8$  potential interactions when 100 noise variables are added.}}
\end{center}
\end{figure}

\begin{figure}
\begin{center}
\includegraphics[width=0.9\textwidth]{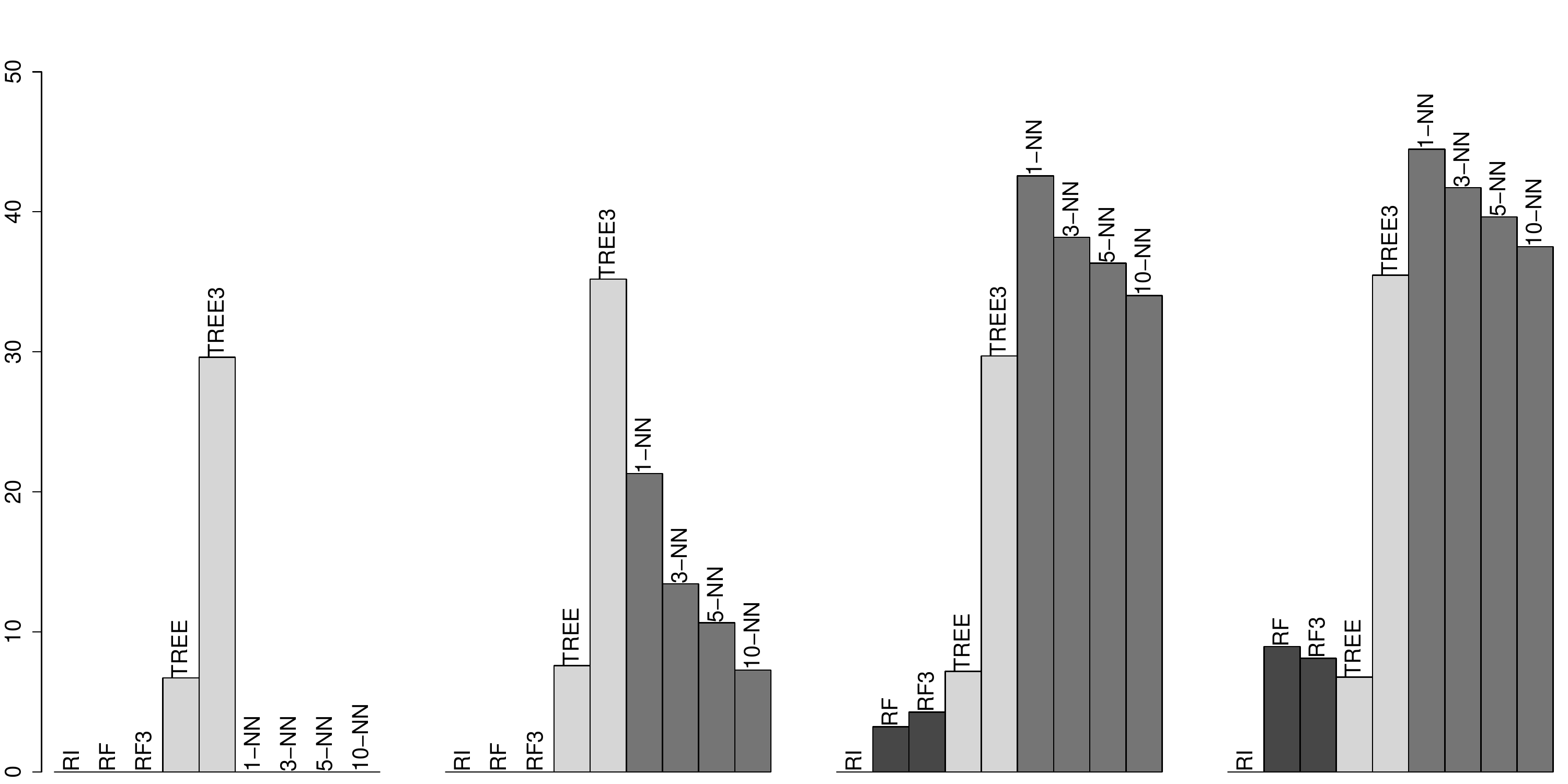}
\caption{ \label{fig:tictac2}  \textit{From left to right: the misclassification rate (in \%) on Tic-Tac-Toe data for 0, 60, 300 and 400 added noise variables. Each classifier is tuned to have equal misclassification rate in both classes. The simple classifier based on {Random Intersection Trees} (RI) has a misclassification rate of 0\% in all cases, as the winning patterns are sampled very frequently (see Figure~\ref{fig:tictac}). {Random Forests} (RF) and {Random Forests} limited to depth 3 trees (RF3) are competitive but the misclassification rate increases sharply when many noise variables are added.}}
\end{center}
\end{figure}

The Tic-Tac-Toe endgame dataset \citep{matheus1989constructive,aha1991instance} contains all possible winning end states of the game Tic-Tac-Toe, along with which player (white or black) has won for each of these. There are just under 1000 possible such end states, and our goal is to learn the rules that determine which player wins from a randomly chosen subset of these. We use half of the observations for training, and the other half for testing.

There are 9 variables in the original dataset which can take the values `black', `white' or `blank'. These can trivially be transformed into a set of twice as many binary variables where the first block of variables encodes presence of black and the second block encodes presence of white.

Two properties of this dataset that make it particularly interesting for us here are: 
\begin{itemize}
\item The presence of interactions is obvious by the nature of the game.
\item There are only very weak marginal effects. Knowing that the upper right corner is occupied by a black stone is only very weakly informative about the winner of the game. Greedy searches by trees fail in the presence of many added noise variables and linear models do not work well at all.
\end{itemize}

We apply \emph{Random Intersection Trees} to finding patterns that indicate a black win (class~1), and also patterns that indicate a white win (class 0). We use the early stopping modifications proposed in Section~\ref{section:stopping}, and create two min-wise hash tables from the available observations in each of the classes, taking $L=200$. Figure \ref{fig:tictactree} shows how the individual Intersection Trees are constructed and illustrates the use of the early stopping rule. We emphasise that we do not need to specify or know that the winning states are functions of only three variables. We let each tree run until all its branches terminate, and collect all resulting leaves.

Figure \ref{fig:tictac} illustrates the importance sampling effect of \emph{Random Intersection Trees} when using only the training data, and adding a varying number of noise variables. When adding 100 noise variables, all 16 winning final combinations are among the 40 most frequently chosen patterns.  All winning states are chosen hundreds of millions times more often than a random sampling of interactions would pick them.

As discussed in Section~\ref{section:intro}, the interactions or rules that are found could be entered into any existing aggregation method, such as \emph{Rule Ensembles} \citep{friedman2005plv}  or \emph{Decision Lists} \citep{marchand2006learning,rivest1987learning}. Here, we consider an even simpler aggregation method by selecting all patterns during $1000$ iterations of \emph{Random Intersection Trees} (with $B=5$ samples as branching factor in each tree) that were selected by at least two trees.
For each selected pattern, we compute the (empirical) class distributions conditional on the presence and absence of the pattern, using the training sample. That is, for each selected pattern $S$, we compute
\begin{gather*}
 \pr_n (Y=1 | X \subseteq S) \; \text{ and } \;\pr_n (Y=1 | X \nsubseteq S).
\end{gather*}
Then, given an observation from the test set, we classify according to the average of the log-odds of being in class~1 calculated from each of the conditional probabilities above.

Figure \ref{fig:tictac2} shows the misclassification rates under situations with different numbers of added noise variables. The simple prediction based on \emph{Random Intersection Trees} achieves perfect classification even when 400 noise variables are added.
Neither $k$-NN nor \emph{CART} \citep{CART}, either restricted to trees of depth 3 (TREE3) or depth chosen by cross-validation (TREE), are as successful, giving misclassification rates between 5\% and 40\%. Interestingly, trees of depth 3 perform much worse than deeper trees. The winning patterns are not identified in a pure form but only after some other variables have been factored in first. 
This also means that it is very hard to read the winning states of the trees, unlike the patterns obtained by our method. \emph{Random Forests} also maintain a 0\% misclassification rate up until about a hundred added noise variables but start to degrade in performance when further noise variables are added. It is easy to identify the noise variables from a variable importance plot \citep{strobl2008conditional}. However,
within the signal variables the patterns are not easy to see since each variable is approximately equally important for determining the winner (with the slight exception of the middle field in the $3 \times 3$ board which is more important than the other fields) and the nature of the interactions is thus not obvious from analysing a \emph{Random Forest} fit.

\subsection{Reuters RCV1 text classification}

\begin{figure}
\begin{center}
\includegraphics[width=0.98\textwidth]{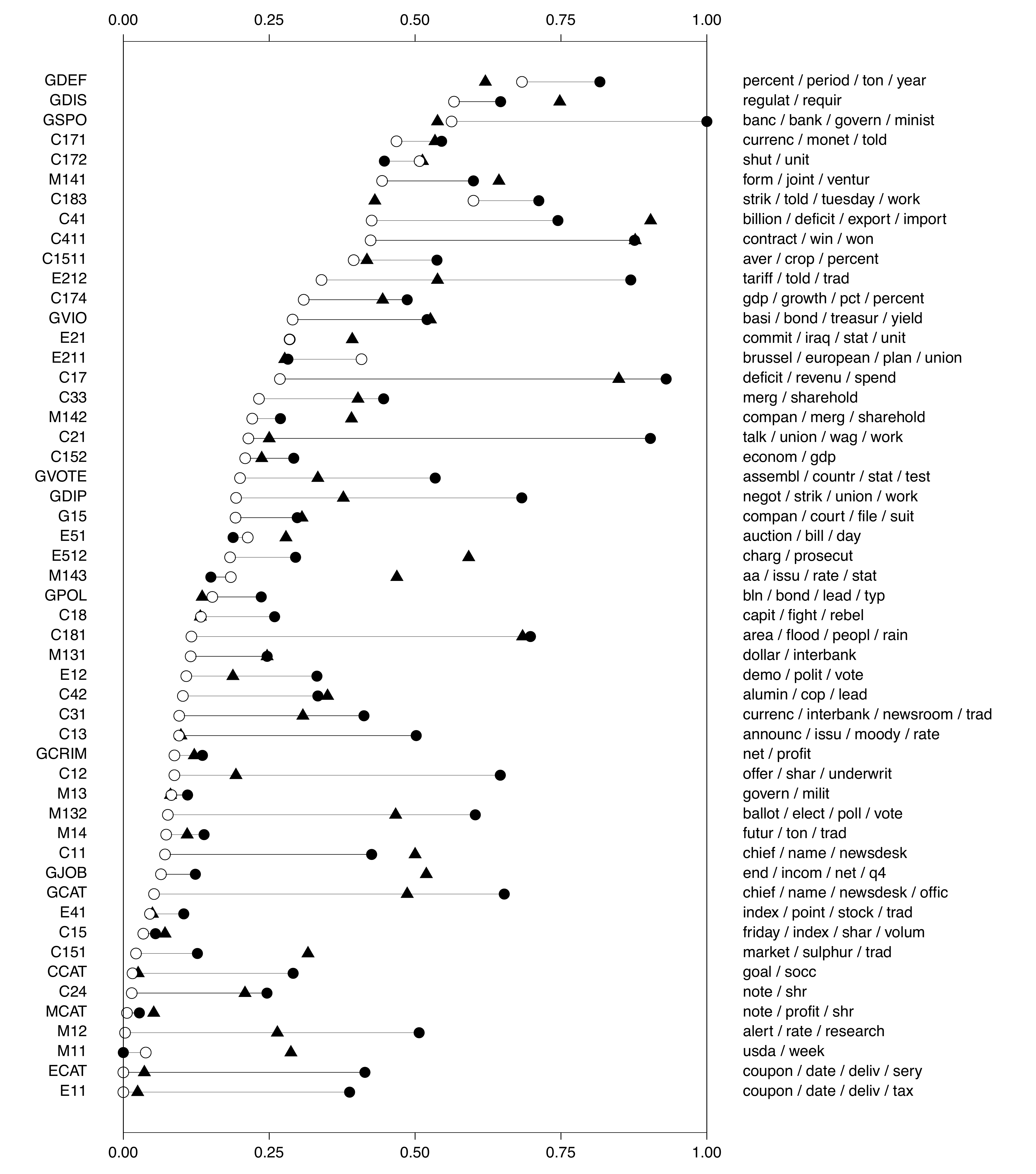}
\caption{ \label{fig:losstext}  \textit{ The misclassification rate $\pr_n(c\notin Y |S\subseteq X)$ on the test data for a patterns $S$ chosen with  a tree ensemble node generation mechanism (black circle), Random Intersections (white circle), and a linear method (black triangle) for topics $c\in C$ in the Reuters RCV1 text classification data. The topics are shown on the left and the word combinations chosen by Random Intersection Trees on the right.  }}
\end{center}
\end{figure}
The Reuters RCV1 text data contain the \emph{tf--idf} (term frequency--inverse document frequency) weighted presence of $47148$ word-stems in each document; for details on the collection and processing of the original data, see \citet{lewis2004rcv1}. Each document is assigned possibly more than one topic. Here we are interested in whether \emph{Random Intersection Trees} is able to give a quick and accurate summary of each topic. For each topic, we seek sets of word-stems, $S$, whose simultaneous presence is indicative of a document falling within that topic.

To evaluate the performance of \emph{Random Intersections}, we divide the documents into a training and test set with the first batch of 23149 documents as training and the following 30000 documents as test documents. We compare our procedure to an approach based on \emph{Random Forests} and a simple linear method.

\emph{Random Forests} and classification trees can be very time- and memory-intensive to apply on a dataset of the scale we consider here. In order to be able to compute \emph{Random Forests}, we only consider word-stems if  they appear in at least  100 documents in the training data. This leaves $2484$ word-stems as predictor variables. We also only consider topics that  contain at least 200 documents. To simplify the problem further, we consider a binary version of the predictor variables for all methods, using a 1 or 0 to represent whether each tf--idf value is positive or not.

Let $C$ be the set of topics in our modified dataset. Let $Y\subseteq C$ indicate the topics that a given document belongs to. Consider a topic or class $c\in C$.
Our goal is to find patterns $S$ that maximise
\begin{equation} \label{eq:opt_reuters}
\pr_n(c\in Y|S\subseteq X), 
\end{equation}
whilst also maintaining that the prevalence of $S$ among all observations be bounded away from 0. Specifically, we shall require that
\begin{equation} \label{eq:constr_reuters}
 \pr_n(S \subseteq X) \geq p_c / 10 \; \text{ where } p_c= \pr_n (c \in Y).
\end{equation}
To see how this can be cast within the framework set in \eqref{eq:opt}, note that if $S^\star$ maximises \eqref{eq:opt_reuters} and $S^{\star\star}$ satisfies
\begin{gather}
 \pr_n (S^{\star\star} \subseteq X | Y \in c) \geq \pr_n (S^{\star} \subseteq X | Y \in c)\; \text{ and} \\
 \pr_n (S^{\star\star} \subseteq X | Y \notin c) \leq \pr_n (S^{\star} \subseteq X | Y \notin c),
\end{gather}
then
\begin{align*}
 \pr_n(c\in Y|S^{\star} \subseteq X) &= \frac{\pr_n (S^\star \subseteq X | c \in Y) \pr_n(c \in Y)}{\pr_n(S^\star \subseteq X | c \in Y) \pr_n (c \in Y)+\pr_n(S^\star \subseteq X |c \notin Y)\pr_n(c \notin Y)} \\
 & \leq \frac{\pr_n (S^{\star \star} \subseteq X | c \in Y) \pr_n(c \in Y)}{\pr_n(S^{\star \star} \subseteq X | c \in Y) \pr_n (c \in Y)+\pr_n(S^{\star \star} \subseteq X |c \notin Y)\pr_n(c \notin Y)} \\
 &= \pr_n (c \in Y | S^{\star \star}),
\end{align*}
whence $S^{\star \star}$ also maximises \eqref{eq:opt_reuters} by optimality of $S^{\star}$. Thus treating those documents belonging to topic $c$ as class~1, and all others as class 0, by solving \eqref{eq:opt} with $\theta_0$ and $\theta_1$ chosen appropriately, we can obtain all solutions to \eqref{eq:opt_reuters}. 

In view of this, we use each of the methods to search for patterns $S$ that have high prevalence for a given topic $c$. We then remove all patterns that do not satisfy \eqref{eq:constr_reuters} on the test data. Then, from the remaining patterns, we select the one that maximises \eqref{eq:opt_reuters} on training data. Below, we describe specific implementation details of each of the methods under consideration.

\paragraph*{Random Intersection Trees}
We create the min-wise hash table for the prevalence among all samples once, using 200 permutations with associated min-wise hash values for each word-stem. Then 1000 iterations of the tree search are performed with a cutoff value $\theta_0= (3/20)p_c$ and all remaining patterns $S$ with a length less or equal to 4 are retained. 

\paragraph*{Random Forests} For a tree-based procedure, one approach is to fit classification trees on subsampled data and adding randomness in the variable selection as in \emph{Random Forests} \citep{breiman01random} and then looking among all created leaf nodes for the most suitable node among all nodes created.

We generate 100 trees as in the \emph{Random Forests} method: each is fit to subsampled training data using \emph{CART} algorithm restricted to depth 4, and further randomness is injected by only permitting variables to be selected from a random subset of those available, for each tree. This takes on average between 90\% to 110\% of the computational time of a non-optimised pure R \citep{R} implementation of \emph{Random Intersection Trees} for these data. Note that this is when using the Fortran version of \citep{breiman01random} for the \emph{Random Forests} node generation; we expect a significant speedup if Fortran or C code were used for \emph{Random Intersection Trees}. We are currently working on such a version and plan to make it  available soon. Furthermore, \emph{Random Forests} would scale much worse if many more word-stems were included as variables.

\paragraph*{Linear models} For linear models, we fit a sparse model with at most $\ell$ predictors (with $\ell \le 4$), using a logistic model with an $\ell_1$-penalty \citep{tibshirani96regression,friedman2010regularization}.  We constrain the regression coefficients to be positive since we are only looking for positive associations in the two previously discussed approaches,  and want to keep the same interpretability for the linear model. For each value of $\ell \le 4$, we take $S_\ell$ to be the  set of variables with a positive regression coefficient. We select the largest value of $\ell $ such that the fraction of documents attaining the maximal value is at least $p_c/10$ and select the associated pattern $S_\ell$. (An alternative approach would be to retain the documents with the highest predicted value when using a sparse regression fit. This approach gave very similar results.)

After screening the candidate patterns returned by each of the methods using \eqref{eq:constr_reuters} on all of the topics $c \in C$, we evaluate the misclassification rate $\pr_n(c\notin C|S\subseteq X)$ on the test data.
The results for all of the topics are shown in Figure~\ref{fig:losstext}. The rules found with \emph{Random Intersection Trees} have a smaller loss than those found with \emph{Random Forests} in all but 5 of the topics. For those topics where \emph{Random Forests} performs better, the difference in loss is typically small. Linear models achieve a smaller loss than \emph{Random Forests} among most of the topics, but only have a smaller loss than \emph{Random Intersection Trees} in 6 topics, performing worse in all remaining 46 topics.

\section{Discussion}
\label{section:discussion}

We have proposed \emph{Random Intersection Trees} as an efficient way of finding interesting interactions. In contrast to more established algorithms, the patterns are not built up incrementally by adding variables to create interactions of greater and greater size. Instead we start from the full interaction $S=\{1,\ldots,p\}$ and remove more and more variables from this set by taking intersections with randomly chosen observations. Arranging the search in a tree increases efficiency by exploiting sparsity in the data. For the basic version of our method (Algorithm 1), we were able to derive a bound on the computational complexity. The bound depends on (a) the prevalence or frequency with which the pattern $S$ appears among observations in class~1, and (b) the overall sparsity of the data, with higher sparsity making it easier to detect the interaction using a given computational budget. In the best case, we can achieve an almost a linear complexity bound as a function of $p$; more generally our complexity bound typically has a smaller exponent than that for 
a brute force search. Further improvements can be made by using min-wise hashing techniques to terminate parts of the search (i.e. branches of the Intersection Tree) that have no chance of leading to interesting interactions. Numerical examples illustrate the improved interaction detection power of \emph{Random Intersection Trees} over other tree-based methods and linear models. 

There are many diverse ways in which interactions that solve \eqref{eq:opt} can be used in further analysis. The interactions may be of interest in their own right as shown in both numerical examples. One can also simply use the search to make sure that a dataset is unlikely to have strong interactions that could otherwise have been missed. If the aim is to build a classifier, they can be added to a linear model, or built into classifiers based on tree ensembles. For the latter approach one could consider, for example, averaging predictions in a linear way or averaging log-odds as in \emph{Random Ferns} \citep{bosch2007image}. We believe developments along these lines will prove to be fruitful directions for future research. We also plan to generalise the idea to categorical and continuous predictor variables. 

\bibliographystyle{plainnat}
\bibliography{biball}

\section{Appendix}
\label{section:appendix}

\paragraph{Proof of Theorem~\ref{thm:1}}
Fix a tree $m \in \{1, \ldots, M\}$ and suppose this has node set $N = \{1, \ldots, J\}$ indexed chronologically (see Section~\ref{section:random}). For $d \in \{1, \ldots, D\}$, define
 \begin{gather*}
  N_d = \{j \in N: \depth(j) = d \; \mbox{and} \; S_j \supseteq S\}, \\
  W_d = |N_d|.
 \end{gather*}
Let $E$ be the event that $S$ is contained in $S_1$, the random sample selected for the root node of tree $m$. Further, let $G_d (t) = \E(t^{W_d} |E)$, the probability generating function of $W_d$ conditional on the event $E$.

We make a few simple observations from the theory of branching processes. Firstly, for $d \leq D-1$, $G_{d+1} = G_d \circ G$ where $G:= G_1$. To see this, first note that
\[
 W_{d+1} = \sum_{j \in N_d} \sum_{j' \in \ch(j)} \ind_{\{S \subseteq X_{i(j')}\}}.
\]
Now conditional on $E$, the random variables $\sum_{j' \in \ch(j)} \ind_{\{S \subseteq X_{i(j')}\}}$ for $j \in N_d$, are independent of $N_d$. Moreover, they are independent of each other and have identical distributions equal to that of
\[
 \sum_{j' \in \ch(1)} \ind_{\{S \subseteq X_{i(j')}\}} = W_1.
\]
This entails
\[
 \E(t^{W_{d+1}} |W_d = w, E) = \{\E(t^{W_1} | E)\}^w = \{G(t)\}^w.
\]
Thus
\[
 G_{d+1} (t) = \E(\E (t^{W_{d+1}} | W_d, E) | E) =  \E(\{G(t)\}^{W_d} | E) = G_d (G(t)),
\]
as claimed.

From this we can conclude that if $G$ has a fixed point $q \in (0, 1)$, then this must be a fixed point for all $G_d$. Since each $G_d$ is non-decreasing, we have that for all $d \in \mathbb{N}$, if $q' \leq q$ then, $G_d (q') \leq q$. The relevance of these remarks will become clear from the following: for an $S' \in L_{D, m}$, we have
\begin{align*}
 G_D (\pr(S' \supsetneq S | S' \supseteq S)) &= \sum_{\ell = 0} ^\infty \pr(W_d = \ell | E) \pr(S' \supsetneq S | S' \supseteq S)^\ell \\
  &= \sum_{\ell = 0} ^\infty \pr(\{W_d = \ell\} \cap \{S \notin L_{D, m}\} |E) \\
  &= \pr(S \notin L_{D,m} | E).
\end{align*}
Thus if we can ensure that $\pr(S' \supsetneq S | S' \supseteq S)$ is at most $q$, then the final probability in the above display will also be at most $q$.

To get an upper bound for $\pr(S' \supsetneq S | S' \supseteq S)$,
we argue as follows. The set $S'$ is the intersection of $D+1$ observations selected independently of one another. In order for some $k \in S^c$ to be contained in $S'$, it must have been present in all these $D+1$ observations. Thus by the union bound we have
\begin{align*}
 \pr(S' \supsetneq S | S' \supseteq S) \;\leq \; \sum_{k\in S^c} \pr ( k \in S' | S' \supseteq S ) \; \leq \; p \nu^{D+1},
\end{align*}
the rightmost inequality following from (A2).

Now let the distribution of $B$ be such that
\begin{equation*}
 B = \begin{cases}
      b \quad & \text{with probability } 1 - \alpha, \\
      b + 1 \quad  & \text{with probability } \alpha.
     \end{cases}
\end{equation*}
Note that
\begin{align*}
 \E(B^d) = \sum_{\ell = 0} ^k \{b(1-\alpha)\}^\ell \{(b+1) \alpha\}^{d - \ell} \binom{d}{\ell} = (b+\alpha)^d
\end{align*}
Using this, we see that the expected computational complexity of the algorithm is bounded above by
\begin{align} \label{eq:complexity_1}
 & \log(p) M \sum_{k=1} ^p  [(b+\alpha)\delta_k + \cdots + \{(b + \alpha)\delta_k\}^D] \notag \notag \\
 & \leq \log(p) M D \bigg[ p  + \sum_{k : (b+\alpha)\delta_k > 1} \big\{\big((b+\alpha) \delta_k \big)^D - 1 \big\} \bigg]. 
\end{align}

Now given a $q \in (0, 1)$, we shall pick $B \in \mathbb{N}$ and $\alpha \in [0, 1)$ to satisfy $G(q) = q$. To this end, observe that
\begin{equation*}
 G(q) = (1-\alpha)(1 - \theta_1 (1-q))^b + \alpha(1 - \theta_1 (1-q))^{b + 1}.
\end{equation*}
Thus $\alpha$ and $b$ must satisfy
\begin{align} \label{eq:b+alpha}
 b + \alpha &= \frac{\log(q) - \log(1 - \alpha \theta_1 (1 - q))}{\log(1 - \theta_1(1-q))} + \alpha \notag \\
 & \leq \frac{-\log(q) + \log(1 - \alpha \theta_1 (1 - q))}{\theta_1(1-q)} + \alpha \notag \\
 & \leq \frac{-\log(q)}{\theta_1 (1 - q)} \notag \\
  & \leq \frac{1 + (1-q)/(2q)}{\theta_1}.
\end{align}
In the final line, we used the inequality
\[
 \log(z) \geq (z-1) - \frac{(z-1)^2}{2z}, \quad 0<z \leq 1.
\]

Next, we pick $D$ to be the minimum $D$ such that $p \nu^{D+1} \geq q$, so
\begin{equation} \label{eq:D_choice}
 D = \ceil{\frac{\log(p/q)}{\log(1/\nu)}} - 1,
\end{equation}
and in particular
\begin{equation} \label{eq:D+1}
 D \leq \frac{\log(p/q)}{\log(1/\nu)}.
\end{equation}

Finally, note that the probability of recovering $S$ is
\begin{equation*}
 1 - [1 - \{1 - \pr(S \notin L_{D,m} | E)\}\theta_1]^M.
\end{equation*}
Given the choices of $\alpha$ and $b$ \eqref{eq:b+alpha}, and $D$ \eqref{eq:D_choice}, we have that $\pr(S \notin L_{D,m} | E) \leq q$. Thus taking $M$ to be at least
\begin{equation} \label{eq:M}
 \frac{-\log(\eta)}{(1-q) \theta_1} \geq \frac{-\log(\eta)}{\log\{1 - (1-q)\theta_1\}}
\end{equation}
guarantees recovery of $S$ with probability at least $1 - \eta$. Substituting equations \eqref{eq:b+alpha}, \eqref{eq:D+1} and \eqref{eq:M} into the complexity bound \eqref{eq:complexity_1}, and writing $\epsilon = (1-q)/(2q)$ gives a bound for the computational complexity of
\begin{align*}
\log(p)\frac{\log(1/\eta)}{\theta_1} \frac{1+2\epsilon}{2\epsilon} \frac{\log\{p(1+2\epsilon)\}}{\log(1/\nu)} \bigg[ p + \sum_{k:(1+\epsilon)\delta_k > \theta_1} \big\{\big(p(1+2\epsilon)\big)^{\frac{\log\{(1+\epsilon)\delta_k / \theta_1\}}{\log(1/\nu)}} - 1 \big\} \bigg].
\end{align*}
Given that $\epsilon$ is bounded above, removing constant factors not depending on $p$, we get that the order of the computational complexity is bounded above by
\[
 \log(1/\eta)\frac{\log^2 (p)}{\epsilon} \bigg\{p + \sum_{k:(1+\epsilon)\delta_k > \theta_1} \big(p^{\frac{\log\{(1+\epsilon)\delta_k / \theta_1\}}{\log(1/\nu)}} - 1 \big) \bigg\}. \qed
\]

\paragraph{Proof of Corollary~\ref{cor:1}}
Note that
\[
 \sum_{k:(1+\epsilon)\delta_k > \theta_1}  p ^{\frac{\log((1+\epsilon)\delta_k/\theta_1 )}{\log(1/\nu)}} 
 \]
is bounded by
\[  p^{\gamma} \cdot p ^{(1+\epsilon) \alpha^\star/\beta } \ind_{\{\alpha^\star/\beta>0\}} \;+\;   p\cdot p ^{(1+\epsilon) \alpha_\star/\beta } \ind_{\{\alpha_\star/\beta>0\}} .\]
The result then follows using the scaling of $p\log^2(p)$ and the possibility of making $\epsilon$ arbitrarily small in Theorem~\ref{thm:1}. \qed

\paragraph{Derivation of~\eqref{eq:expect_hash}}
Writing $r = n \pi_2 (S)$, we have
\begin{align*}
 \binom{n}{r} \E_\sigma (\min_{k \in S} h_\sigma (k)) & = \sum_{\ell = 1} ^{n - r + 1} \ell \binom{n-\ell}{r-1} \\
 & = \sum_{\ell = 1} ^{n-r +1} \left\{(\ell - 1)\binom{n - (\ell - 1)}{r} - \ell \binom{n-\ell}{r}\right\} + \sum_{\ell = 1} ^{n-r+1} \binom{n-\ell+1}{r} .
\end{align*}
The first two terms sum to zero leaving only the final term. Thus
\begin{align}
 \binom{n}{r} \E_\sigma (\min_{k \in S} h_\sigma (k)) & = \sum_{\ell = 1} ^{n-r+1} \left\{ \binom{n - \ell + 2}{r+1} - \binom{n-\ell +1}{r+1} \right\} \notag \\  \label{eq:expect_hash1}
 &= \binom{n+1}{r+1}, 
\end{align}
whence
\begin{equation} \label{eq:expect_hash2}
 \E_\sigma (\min_{k \in S} h_\sigma (k)) = \frac{n + 1}{r + 1}. \qed
\end{equation}

\paragraph{Proof of Theorem~\ref{thm:minhash_asymp}}
Writing
\[
 \tilde{\pi}_2 ^{-1} (L; S, H) := \tfrac{1}{L}\sum_{l=1}^L  \min_{k\in S} H_{l k}
\]
and suppressing dependence on $S$ and $H$, we have
\begin{align}
 \hat{\pi}_1 \hat{\pi}_2 - \pi_1 \pi_2 &= \frac{(n + 1 - \tilde{\pi}_2 ^{-1}) \hat{\pi}_1 }{n\tilde{\pi}_2 ^{-1}} - \pi_1 \pi_2 \notag \\ \label{eq:asymp_expan}
 & = \frac{n + 1 - \tilde{\pi}_2 ^{-1}}{n \tilde{\pi}_2 ^{-1}}\left\{(\hat{\pi}_1 - \pi_1) - \pi_1 \frac{n \pi_2 + 1}{n + 1 - \tilde{\pi}_2 ^{-1}} \left(\tilde{\pi}_2 ^{-1} - \frac{n+1}{n \pi_2 + 1} \right) \right\}.
\end{align}
Consider $L \to \infty$. By the weak law of large numbers and the continuous mapping theorem, we have
\begin{gather*}
 \frac{n + 1 - \tilde{\pi}_2 ^{-1} (L)}{n \tilde{\pi}_2 ^{-1} (L)} \inprob \pi_2 \quad \text{and} \\
 \frac{n\pi_2 + 1}{n + 1 - \tilde{\pi}_2 ^{-1} (L)} \inprob \frac{(\pi_2 + n^{-1})^2}{\pi_2 (1 + n^{-1})}. 
\end{gather*}
By the central limit theorem, Slutsky's lemma and Lemma \ref{lem:var},
\begin{gather*}
 A_L := \sqrt{L} (\hat{\pi}_1 (L) - \pi_1) \indist N(0, \pi_1 (1 - \pi_1)) \quad \text{and} \\
 B_L := -\pi_1 \frac{n\pi_2 + 1}{n + 1 - \tilde{\pi}_2 ^{-1} (L)} \times \sqrt{L} \left(\tilde{\pi}_2 ^{-1} (L) - \frac{n+1}{n \pi_2 + 1} \right) \indist N(0,\pi_1 ^2( 1 - \pi_2) + o_n(1)).
\end{gather*}
Now $\hat{\pi_1}$ and $\tilde{\pi_2}^{-1}$ are independent, so $A_L$ and $B_L$ are independent. Thus we have that for all $t_1, t_2 \in \R$,
\[
 \E(e^{i(t_1 A_L + t_2 B_L)}) = \E(e^{it_1 A_L}) \E(e^{it_2 B_L}) \to \exp\{\tfrac{1}{2}( t_1 ^2 \pi_1 (1 - \pi_1) + t_2 ^2 (\pi_1 ^2( 1 - \pi_2) + o_n(1))\}.
\]
pointwise as $L \to \infty$. Returning to \eqref{eq:asymp_expan}, by L\'{e}vy's continuity theorem we have
\begin{equation*}
 \sqrt{L}(\hat{\pi}_1 (L) \hat{\pi}_2 (L) - \pi_1 \pi_2) \indist N(0, \pi_2 ^2 \pi_1 (1 - \pi_1 \pi_2) + o_n(1)). \qed
\end{equation*}

\begin{lem} \label{lem:var}
Let $r = n \pi_2 (S)$ and suppose $n \geq r+2$. Then
 \begin{align*}
  \Var_\sigma (\min_{k \in S} h_\sigma (k)) & = \frac{(rn^2 -3r^2 n + 2r^3) + (rn + 4n + 5r^2 + 4r + 2)}{(r+1)^2 (r+2)}.
 \end{align*}
\end{lem}
\begin{proof}
We have,
\begin{align*}
 \binom{n}{r} \E_\sigma \{ (\min_{k \in S} h_\sigma (k))^2\} & = \sum_{\ell = 1} ^{n - r + 1} \ell^2 \binom{n-\ell}{r-1} \\
 & = \sum_{\ell = 1} ^{n - r + 1} \left\{ (\ell-1)^2 \binom{n-\ell}{r} - \ell^2 \binom{n-\ell}{r}\right\} \\
 & \qquad + \sum_{\ell=1} ^{n-r+1} \left\{ 2(\ell-1)\binom{n-\ell+1}{r} + \binom{n-\ell+1}{r}\right\} \\
 &= 2\binom{n}{r+2} + \binom{n+1}{r+1},
\end{align*}
where in the last line we used \eqref{eq:expect_hash1} and \eqref{eq:expect_hash2}. From this, we get that
\begin{align*}
 \Var_\sigma (\min_{k \in S} h_\sigma (k)) &= \frac{(rn^2 -3r^2 n + 2r^3) + (rn + 4n + 5r^2 + 4r + 2)}{(r+1)^2 (r+2)}. \qedhere
\end{align*}
\end{proof}
\end{document}